\documentclass[sigconf]{acmart-arxiv}
 \AtBeginDocument{%
   \providecommand\BibTeX{{%
     \normalfont B\kern-0.5em{\scshape i\kern-0.25em b}\kern-0.8em\TeX}}}

\usepackage{microtype}
\usepackage{graphicx}
\usepackage{subfigure}
\usepackage{amsthm}
\usepackage{amsmath}
\usepackage{url}            % simple URL typesetting
\usepackage{booktabs}       
\usepackage{amsfonts}       % blackboard math symbols
\usepackage{nicefrac}       % compact symbols for 1/2, etc.
\usepackage{lipsum}
\usepackage{algorithm}
\usepackage{algorithmic}
\usepackage{hyperref}
\usepackage{tabularx}
\usepackage{makecell}
\usepackage{amsthm}
\newtheorem{theorem}{Theorem}

\newcommand\independent{\protect\mathpalette{\protect\independenT}{\perp}}
\def\independenT#1#2{\mathrel{\rlap{$#1#2$}\mkern2mu{#1#2}}}

\newtheorem{assump}{Assumption}

\theoremstyle{definition}
\newtheorem{definition}{Definition}

\begin{document}

\title{Counterfactual fairness: removing direct effects through regularization}
\author{Pietro G. Di Stefano}
\authornote{To whom correspondence should be addressed}
\email{pietro.distefano@experian.com}
\affiliation{%
\institution{Experian UK\&I and EMEA DataLabs}
\streetaddress{110 Bishopsgate EC2N 4AY}
\city{London}
\state{UK}
}
\author{James M. Hickey}
\email{james.hickey@experian.com}
\affiliation{%
\institution{Experian UK\&I and EMEA DataLabs}
\streetaddress{110 Bishopsgate EC2N 4AY}
\city{London}
\state{UK}
}
\author{Vlasios Vasileiou}
\email{vlasios.vasileiou@experian.com}
\affiliation{%
\institution{Experian UK\&I and EMEA DataLabs}
\streetaddress{110 Bishopsgate EC2N 4AY}
\city{London}
\state{UK}
}

\begin{abstract}
Building machine learning models that are \textit{fair} with respect to an unprivileged group is a topical problem. Modern fairness-aware algorithms often ignore causal effects and enforce fairness through modifications applicable to only a subset of machine learning models. In this work, we propose a new definition of fairness that incorporates causality through the Controlled Direct Effect (CDE). We develop regularizations to tackle classical fairness measures and present a causal regularization that satisfies our new fairness definition by removing the impact of unprivileged group variables on the model outcomes as measured by the CDE. These regularizations are applicable to any model trained using by iteratively minimizing a loss through differentiation. We demonstrate our approaches using both gradient boosting and logistic regression on: a synthetic dataset, the UCI Adult (Census) Dataset, and a real-world credit-risk dataset. Our results were found to mitigate unfairness from the predictions with small reductions in model performance. 
\end{abstract}




\maketitle

\section{Introduction}
\label{intro}
One of the most famous concepts in computer science is the one of ``garbage in, garbage out''. Applied to machine learning algorithms, this phrase captures the concept that the ability to train and the quality of the output from a machine learning model is dependent on the quality of the training data presented to it. Advances in recent decades have resulted in machine learning algorithms that can leverage a larger variety of data types and sources for ever more complex learning tasks. This increased capacity to learn from data also raises risks of incorporating undesired biases from the training data into a machine learning model~\cite{Framework2019, AutomatedPredictionLaw2012}. Furthermore, even when the data is completely unbiased, machine learning systems can produce biased results for specific groups, as can the case where the sensitive groups form a small minority of the training data. This problem goes beyond ``bias in, bias out'' becoming in the worst case ``fairness in, bias out''. It has presented itself in a range of domains from credit-risk assessment to determining an individual's propensity for criminal recidivism~\cite{PredictAndServe2016, Recidivism2017}, and has attracted the attention of both regulators and the media~\cite{WeaponsMathDestruction2016}. 

After the ``fairness through unawareness'' paradigm was shown to be flawed~\cite{PhantomMenace2005, Dwork2012}, new approaches to handle discrimination in machine learning have been explored. Recent advances enable practitioners to specify which groups, usually derived from one or more sensitive attributes, they are concerned about possibly treating unfairly and design a system that ameliorates potential bias (unfairness) in the outputs.  These adaptations cover the full pipeline of the machine learning training problem, with the main approaches including pre-processing of the data~\cite{Reweighing2012, DIRemover2015, OptPreProc2017, Dwork2013}, post-processing of outputs~\cite{EqOddsPP2016, CalibEqOdds2017, RejectOptClf2012} and incorporation of fairness constraints directly into the objective function~\cite{IBM2018, FNN2018, Goel2018, MetaLearn2019, NaiveBayesFair2010, FairRed2018}. Unfortunately, all of these methods suffer from drawbacks~\cite{FairnessImpossibility2018, DesigningFairAlgos2018} such as requiring access to sensitive attributes even after model training (an undesirable scenario in many circumstances), ignoring causal structures in the data, and being specific to a particular machine learning algorithm.

These drawbacks pose serious issues for anyone wishing to ensure fairness in their machine learning practices. Firstly, the wide diversity of statistical fairness measures makes it difficult to select an appropriate metric and corresponding fairness-aware algorithm. This is exacerbated by the fact that while many metrics seemingly address the same underlying notion of fairness, those metrics cannot be mathematically optimized simultaneously for a given task~\cite{impossibility}. Furthermore, the statistical nature of the metrics make it difficult to discern correlation from causation when examining decisions and addressing fairness bias. The role of causality~\cite{Causality2009, Chiappa2018, Kusner2017} when reasoning about fairness should not be understated. It is considered of serious importance by social-choice theorists and ethicists and it has been argued that causal frameworks would decisively improve reasoning about fairness~\cite{loftus2018causal}. This has resulted in several causal definitions of fairness that take advantage of the concept of \textit{counterfactuals} to address the potential unfair \textit{causal effects} on model outcomes. However, the implementation of these causal worldviews typically focused on generative models and not on how causality may be incorporated into popular discriminative machine learning models. Moreover, there is a general lack of comparison between models constrained by statistical fairness metrics and causal effects. 

In this paper, we address each of the presented issues in turn within a \textit{Counterfactual Fairness} framework. In this framework, we first describe a novel definition of fairness that seeks to remove the controlled direct effect (CDE) of the sensitive attribute from the model's predictions. We then propose to enforce such definition through regularization. This is achieved through propensity-score matching~\cite{PropensityScore1983, PropensityScore1985} and the development of a mean-field theory. Our fairness-via-regularization approach is applicable to any model trained by minimizing a loss function through differentiation. We will focus on the case of binary classification, with a single binary indicator defining un-/privileged groups, and are exemplified using gradient boosted trees and logistic regression. 

To further allow evaluation of our Counterfactual Fairness regularization framework, we compare its results against a regularization strategy aimed at satisfying full equality of outcomes between groups. This allows us to highlight how the differences between these worldviews manifest in the optimized model.

The effectiveness and drawbacks of our methods are illustrated using a public benchmark dataset, a private commercial credit-risk dataset, and a synthetic dataset. The inclusion of a private commercial credit-risk dataset provides insight into how these approaches could work in an industrial setting. We consider this form of evaluation to be of critical importance, as the main risks of machine learning fairness are beared by end-consumers and businesses through high-velocity decisioning systems built on such datasets.

The structure of the papers is as follows: in Section \ref{sec:notation} we introduce our notation, in Section \ref{sec:back-caus} we provide a background on some key concepts of causality and mediation effects, in Section \ref{background} we provide a discussion on algorithmic fairness, while in Section \ref{fairboost} we present our regularization strategies. We discuss the results of the experiments in Section \ref{experiments} and relationships with existing literature in Section \ref{sec:related}. Finally, we state our conclusions in Section \ref{conclusions}.

\section{Notation and Setting}
\label{sec:notation}
Before beginning our discussions on how fairness is measured and the connection to causal modelling, we introduce our notation. We use $Y$ to refer to the observed label in the data while the sensitive attribute is denoted $Z$. Throughout this work, our groups of interest will be defined by a single, binary sensitive attribute and we only consider binary classification tasks, i.e. $Y, Z \in \{0,1\}$. We denote the covariates that do not define the groups of interest (the insensitive covariates) by $\mathbb{X}$. We denote probabilities with $P(\bullet)$ and probability densities with $p(\bullet)$. We assume that our models provide probability estimates $\tilde{Y} = P(Y = 1|\mathbb{X})$, and $\hat{Y} \in \{0, 1\}$ are the binary outcomes obtained by thresholding $\tilde{Y}$.

\section{Background on Causality}
\label{sec:back-caus}
\begin{figure}[!t]
    \centering
        \includegraphics[width=0.6\linewidth]{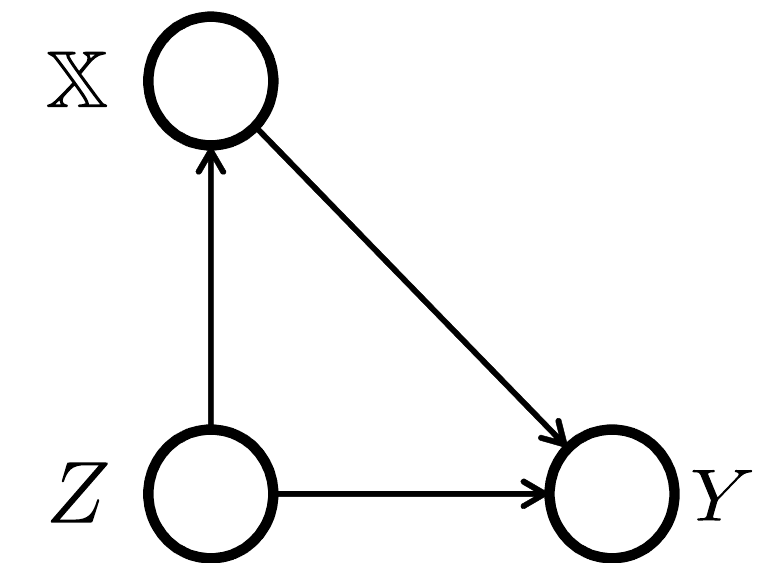}
    \caption{\label{fig:graphs}
Assumed causal graph for our data generating process, where $Z$ is a parent attribute potentially having a direct effect and an indirect effect through a vector-valued set of features $\mathbb{X}$.}
    \label{fig:my_label}
\end{figure}

The aim of this section is to introduce counterfactual quantities and causal effects, especially mediated ones. We follow the causality literature in denoting  counterfactual quantities using subscripts, i.e. we define the counterfactual outcome $Y_{(Z=Z^*,~\mathbb{X} = \mathbb{X}^*)}$ as the outcome that would have been realized had an individual been assigned the values $Z = Z^*$ and $\mathbb{X} = \mathbb{X}^*$. We assume an unconfounded causal graph of the form of Fig. \ref{fig:graphs}, which embeds our assumption of sequential ignorability:

\begin{assump}
 \label{as: str_ign}
 (Sequential ignorability) The following conditional independence relations hold for each realization $Z^*$ and $\mathbb{X}^*$:
  \begin{align}
  Y_{Z=Z^*} \independent& Z \\ \nonumber
  Y_{\mathbb{X}=\mathbb{X}^*} \independent& \mathbb{X} | Z  
 \end{align}
\end{assump}

The former of the above conditions is equivalent to assuming the absence of confounders that would open a ``backdoor'' path between $Z$ and $Y$~\cite{Causality2009}. We argue that for our purposes the absence of such paths is justifiable in a wide variety of cases, as the sensitive attribute is usually measured at birth, and is therefore naturally a parent variable.

To avoid cases in which the protected attribute is completely specified by the mediators, we also assume the following:

\begin{assump}
 \label{as:overlap}
 (Strong ignorability) There is overlap between the two groups: $0 < P(Z=1|\mathbb{X}) < 1,\  \forall \mathbb{X}$.
\end{assump}

We start by defining the Average Treatment Effect (ATE):

\begin{align}
 \label{eq:ate}
 \mathrm{ATE} &= \mathbb{E}[Y_{Z=1} - Y_{Z=0}] \\ \nonumber
              &= \mathbb{E}[Y|Z=1] - \mathbb{E}[Y|Z=0],
\end{align}
where the second line is a consequence of Sequential ignorability and the law of counterfactuals \cite{Causality2009} $P(Y_{Z=Z^*}|Z=Z^*) = P(Y|Z=Z^*)$.

The ATE represents the total causal effect of the protected attribute. In this work, as we shall see in the next section, we are interested in \textit{mediation effects}, i.e. direct and indirect effects, instead of the ATE. We define the point-wise Controlled Direct Effect (CDE) as:

\begin{align}
\label{eq:CDE}
 \mathrm{CDE}(\mathbb{X}) =&  \mathbb{E}[Y_{(Z=1, \mathbb{X})} - Y_{(Z=0, \mathbb{X})}] \\ \nonumber
 =& \mathbb{E}[Y|Z=1, \mathbb{X}] - \mathbb{E}[Y|Z=0, \mathbb{X}],
\end{align}
where, as before, the second line follows from Sequential ignorability. The CDE represents the effect of changing the value of the protected attribute while keeping the value of the covariates fixed. Eq.~\ref{eq:CDE} provides a means for directly estimating the CDE from the data, e.g., via regression techniques~\cite{BaronKenny, VanderWeele2009}.

Another important mediation effect is the Natural Direct Effect (NDE), defined as follows. Given a ``baseline'' value $Z^*$ for the protected attribute, e.g. ``female'' for a gender discrimination problem, we define the NDE as the difference in $Y$ that would be attained by changing $Z$, with the value of the mediating variables $\mathbb{X}$ set to what they would have attained had $Z$ been $Z^*$, i.e. $\mathbb{X}_{Z^*}$. This yields:
\begin{align}
\label{eq:NDE}
 \mathrm{NDE}(Z^*) &= \mathbb{E}[Y_{(Z= 1 - Z^*, \mathbb{X} = \mathbb{X}_{Z^*})} - Y_{Z=Z^*}] \\ \nonumber
                    &= (-1)^{(Z^* - 1)} \int \mathrm{d}\mathbb{X} \ \mathrm{CDE}(\mathbb{X}) p(\mathbb{X}|Z=Z^*),
\end{align}
where the second line is proved in Ref. \cite{Pearl2001}. The natural direct effect has a few nice properties. First, it has straightforward implications in discrimination problems. Second, contrary to the CDE case, it has an indirect counterpart, namely the Natural Indirect Effect (NIE), defined as:

\begin{equation}
 \label{eq:cie}
 \mathrm{NIE}(Z^*) = \mathbb{E}[Y_{Z=Z^*} - Y_{(Z=Z^*, \mathbb{X} = \mathbb{X}_{1 - Z^*})}].
\end{equation}
The NIE is easily interpreted as the effect on the baseline population's response of changing $\mathbb{X}$ to the value it would have naturally obtained in the non-baseline population while keeping the value of the protected attribute fixed.

Natural direct and indirect effects are, in contrast with the CDE, population-wide quantities and depend on the choice of a baseline value for the protected attribute. 

We conclude this section by citing two equalities derived by Pearl \cite{Pearl2001}, which will be useful later and describe the relation between ATE and Natural mediated effects:

\begin{equation}
  \label{eq:ate-cie}
\begin{cases}
  \mathrm{ATE} &= \mathrm{NIE}(1) - \mathrm{NDE}(0) \\ 
 \mathrm{ATE} &= \mathrm{NDE}(1) - \mathrm{NIE}(0) 
\end{cases}
\end{equation}

\section{Algorithmic Fairness}
\label{background}

\subsection{Statistical Fairness}
\label{algo_fairness}
The traditional definitions of algorithmic fairness are statistical in nature and reflect underlying worldviews on how the ``true'' unmeasured target is recorded, its relationship to $Z$ and how it is predicted by a machine learning model. In this framework, statistical measures are derived from one's belief of the latent space structure~\cite{impossibility} rather than by specifying causal pathways between the unobserved ``true'' label and the collected data. Under this construct, two prevailing worldviews of statistical fairness emerge: ``We're all equal'' and ``What you see is what you get''. The former focusses on outcomes defined by group membership and seeks to ensure groups are treated equally through balanced outcomes. Contrastingly, the latter favours is a worldview where the data is accurate and so it seeks to offer similar individuals similar outcomes as informed by the data. We focus here on one of the most popular fairness metrics, namely statistical parity difference. 

SPD is a group fairness measure on an algorithm's outcome, $\hat{Y}$, and it is $0$ (maximally fair) only when $P[\hat{Y}=1|Z=1] = P[\hat{Y}=1|Z=0]$. It is defined as follows:
\begin{equation}
\label{eq:spd}
 \mathrm{SPD} = |P[\hat{Y}=1|Z=1] - P[\hat{Y}=1|Z=0]|.
\end{equation}
We note that this measure can also be applied to the data by changing $\hat{Y}$ to $Y$ in Equation \ref{eq:spd}. As there is no link between the algorithm outcome and measurement, this difference can be minimized by changing the outcomes of members of each group independent of all other attributes and so can be viewed as a ``lazy penalization''.

\subsection{Causal Modelling and Fairness}
\label{sec:counterfactual-fairness}
The measures of fairness presented in the Section \ref{algo_fairness} emerge from \textit{a priori} worldviews on how the underlying ``true'' label is related to $Z$ and the veracity of the recorded data. They do not formally encode the causal relationships between the $Z$, $\mathbb{X}$ and $Y$. Consequently, how much causal effect on $\hat{Y}$ can be attributed to $Z$, either directly or indirectly, is handled implicitly in the worldview employed.

The need to explicitly distinguish between direct and indirect effect is well illustrated by the 1973 University of California, Berkeley gender discrimination scandal~\cite{Berkeley}. In that case, the data showed significant bias in admissions for male and female applicants. However, after controlling for the department chosen by the applicants, that bias disappeared. It was actually found that female applicants had lower overall admission rates not because they were discriminated against, but simply because they were applying to more competitive departments. 

In the Counterfactual Fairness worldview we examine here, we are only concerned with biases that are consequences of \textit{direct effects}. Specifically, we seek to identify, and correct for, how much an outcome for an individual assigned a specific value of $Z$ would change compared to a \textit{counterfactual} world where they had been assigned the alternative value for $Z$ but all other factors had remained the same~\cite{EmplCounterfDiscr}.

This worldview  requires in-depth understanding of the data collection and generation process. In particular, it is important to define what information can be recorded in $\mathbb{X}$. We require that variables $\mathbb{X}$ are ``fair'' in the sense specified by the following requirements:

\begin{enumerate}
 \item They were not measured before $Z$.
 \item They do not, either individually or in combination, directly measure discriminatory attributes.
 \item They are relevant to the problem at hand.
\end{enumerate}

For example, in a financial application where the sensitive attribute is race, ``income'' would be typically considered fair while ``race of the applicant's mother'' or ``blood pressure'' would not be. Taking all of this together, we propose the following definition:

\begin{definition}
\label{def:counterfactual_fairness}
(Counterfactual Fairness) Given a set of fair covariates $\mathbb{X}$, a sensitive attribute $Z$ and a target $Y$, a fair model is a model that does not learn the controlled direct effect of $Z$ on $Y$.
\end{definition}

We stress that any effect mediated by unobserved variables or variables not included in $\mathbb{X}$ during training will be embedded in the CDE and hence the $\mathbb{X}$ will determine the size of the bias we seek to remove.

We finally note that our definition can be seen as an instance of path-specific counterfactual fairness~\cite{ZhangDirectEffect, FairInference2018, Chiappa2018, CFGAN} with the additional requirements on the covariates $\mathbb{X}$ highlighted above.

\section{Fair Losses}
\label{fairboost}
To improve the fairness aspects of a model's output, in both the Statistical Fairness (Sec.~\ref{algo_fairness}) and Counterfactual Fairness (Sec.~\ref{sec:counterfactual-fairness}) worldviews, we modify the loss function and apply regularization. The modifications to the loss function can be applied to all algorithms trained by iteratively minimizing a loss through differentiation, e.g. through gradient descent or gradient boosting.

We combine an original loss function ($\mathcal{L}_{o}$), which captures our utility objectives, with a fairness penalty ($\mathcal{R}_{f}$) using a regularization weight ($\lambda$). The generic form of fairness-aware loss function ($\mathcal{L}_f$) is as follows:
\begin{equation}
 \label{convex_opt}
 \mathcal{L}_f = (1 - \lambda) \mathcal{L}_{o} + \lambda \mathcal{R}_{f},
\end{equation}
where the modifications $\mathcal{R}_{f}$ are differentiable. As we're focussing on binary classification, we take $\mathcal{L}_{o}$ to be the binary cross-entropy of $Y$ and $\tilde{Y}$.

It is important to note that $Z$ is only required at training time to evaluate $\mathcal{R}_{f}$ and is not required for prediction. This is a very useful feature for real-world applications, as obtaining $Z$ can be difficult.

\subsection{Statistical Fairness Regularization}
\label{fairboost:traditional}
For illustrative purposes and to compare our causal worldview with the statistical ``we are all equal'' one, we propose a very simple loss aimed at reducing the SPD of the average scores:

\begin{equation}
\label{eq:spd-loss}
\mathcal{R}^{\mathrm{SPD}}_{f} = \{\mathbb{E}[\tilde{Y}|Z=1] - \mathbb{E}[\tilde{Y}|Z=0]\}^2
\end{equation}

\subsection{Counterfactual Fairness Regularization}
\label{fairboost:causal}
To satisfy Definition~\ref{def:counterfactual_fairness}, the CDE must not be learned by the model. The point-wise CDE is given by Eq. (\ref{eq:CDE}), which can be estimated using regression techniques~\cite{BaronKenny, VanderWeele2009}. However, our algorithm, as we shall see, would require us to fit such regressions iteratively, which can be computationally inefficient if we condition on the full set of covariates. To circumvent this possible computational bottleneck, we employ a ``balancing score''. Balancing scores are functions $b(\mathbb{X})$ of the covariates that make $\mathbb{X}$ and $Z$ conditionally independent, i.e.:
\begin{equation}
\label{eq:bal-scores}
 \mathbb{X} \independent Z | b(\mathbb{X}).
\end{equation}
In our work we use one such score, namely the ``propensity score'' $b(\mathbb{X}) = P(Z=1|\mathbb{X})$. Traditionally this score is utilized in scenarios where $\mathbb{X}$ play the role of confounders rather than mediators. Although the latter case is true in our instance, utilizing the propensity score solves our key computation problem whenever Assumptions 1 and 2 are true.

Given assumptions 1 and 2, we can define a ``mean field'' CDE given by:

\begin{align}
\label{eq:CDEb}
 \mathrm{MFCDE}(b) &= \mathbb{E}[Y|Z=1, b] - \mathbb{E}[Y|Z=0, b] \\ \nonumber
                    &= \frac{\int_{\mathcal{X}_b} \mathrm{d}\mathbb{X}\  p(\mathbb{X}) \mathrm{CDE}(\mathbb{X})}{\int_{\mathcal{X}_b} \mathrm{d}\mathbb{X}\  p(\mathbb{X})},
\end{align}
where $\mathcal{X}_{b^*} := \{ \mathbb{X}\  \mathrm{s.t.}\  b(\mathbb{X}) = b^*\}$. The second line of Eq. (\ref{eq:CDEb}) is proven in Appendix \ref{app:cde}. The MFCDE can thus be interpreted as an average of the CDE across a volume with constant $b=b^*$. We also note that, as in a classic result by Rosenbaum and Rubin~\cite{PropensityScore1983} which we extend to the CDE in Appendix \ref{app:cde}, the population-wide averages of MFCDE and CDE are the same, i.e.

\begin{equation}
 \label{eq:rubin}
 \int \mathrm{d}b\  p(b)\  \mathrm{MFCDE}(b) = \int \mathrm{d} \mathbb{X}\   p(\mathbb{X}). \mathrm{CDE}(\mathbb{X})
\end{equation}

\begin{figure*}[t!]
\centering
  \hspace{0.05\linewidth}
  \includegraphics[width=0.8\linewidth]{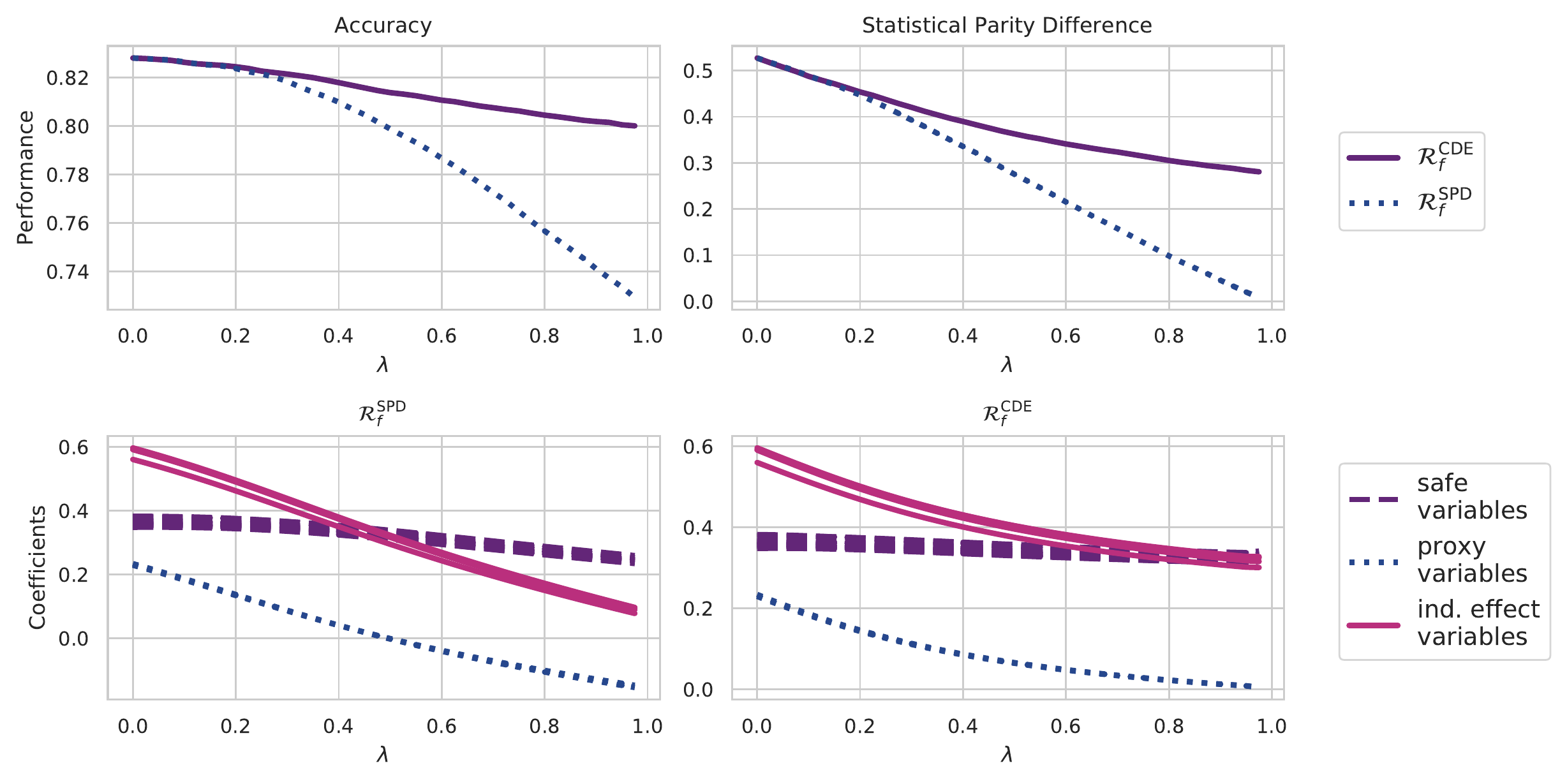}
 \caption{\label{fig:synth}
Results for the synthetic dataset and logistic regression models using the SPD Loss of Eq. (\ref{eq:spd-loss}) and the CDE Loss of Eq. (\ref{eq:cde-loss}). For the CDE loss, we used $N_1=1, N_2=0$. All the results are plotted against the regularization strength $\lambda$. Top panel: Accuracy (left) and SPD (right) using a threshold equal to $0.5$. Bottom panel: Coefficients of the logistic regression models for the safe, indirect effect and proxy variables (see Section \ref{sec:synth-ds})}
\end{figure*}

A stronger statement can be made about the MFCDE if we assume the following:

\begin{assump}
 \label{as:mean-field}
 (Mean field approximation) Wherever $p(\mathbb{X}) > 0$, $\mathrm{CDE}(\mathbb{X})$ is approximately constant in $\mathcal{X}_{b(\mathbb{X})}$.
\end{assump}

This allows us to approximate the MFCDE with the point-wise CDE, i.e. $\mathrm{CDE}(\mathbb{X}) \simeq \mathrm{MFCDE}(b(\mathbb{X}))$. A key result is then that minimization of the CDE (per Counterfactual Fairness) is approximately equivalent to minimizing the MFCDE. Thus, our goal is now to iteratively train a model conditioned on the fair covariates $\mathbb{X}$ only and which does not learn the MFCDE. 

As a first step, we use our training data to estimate $\mathbb{E}[Y|Z, b]$ with a regression model such as:
\begin{equation}
\label{eq:y_b_z}
 \mathbb{E}[Y|Z, b] = \sum_{k=0}^{N_1}\ \alpha_k b^k + Z\ \sum_{k=0}^{N_2}\ \beta_k b^k,
\end{equation}
for arbitrary $N_1$ and $N_2$. We then use this to define a ``fair'' target $Y_f$ such that:

\begin{align}
 \label{eq:yf}
 \mathbb{E}[Y_f|Z, b] &= \mathbb{E}[Y|Z, b] + \frac{(-1)^Z}{2} \mathrm{MFCDE}(b) \\ \nonumber
            &= \sum_{k=0}^{\max(N_1, N_2)} \gamma_k b^k,
\end{align}
where we used the fact that $\mathrm{MFCDE}(b) = \sum_{k=0}^{N_2}\ \beta_k b^k$ and, employing the indicator function $I(\bullet)$, we defined

\begin{equation}
 \label{eq:gammas}
 \gamma_k = \alpha_k\  I(k \le N_1) + \frac{1}{2} \beta_k\  I(k \le N_2).
\end{equation}

$Y_f$ is easily interpreted as a version of a target corrected by a symmetrized version of the CDE, where privileged and unprivileged groups receive anti-symmetrical corrections.

We can also show (see Appendix \ref{app:cde}) that:

\begin{equation}
\label{eq:yf_nie}
 \mathbb{E}[Y_f|Z=1] - \mathbb{E}[Y_f|Z=0] = \frac{\mathrm{NIE}(1) - \mathrm{NIE}(0)}{2},
\end{equation}
meaning that the ATE of $Z$ on $Y_f$ is equal to symmetric version of the indirect effect on $Y$, thereby justifying the definition of fair target.

Having defined our ``fair'' target, and how it relates to the MFCDE (hence, to the CDE), we now describe the procedure to remove the CDE during model training. Firstly, at each iteration of the optimization algorithm, we extract the probability scores estimated by our model at that iteration, the balancing scores and the protected attributes for each training example to define a surrogate training set $\{ (\tilde{Y}_i, b_i, Z_i)\}$. Then, we use this latter to estimate the coefficients of surrogate model:

\begin{equation}
\label{eq:y_hat_b_z}
 \mathbb{E}[\tilde{Y}|Z, b] = \sum_{k=0}^{\max(N_1, N_2)}\ \tilde{\alpha}_k b^k + Z\ \sum_{k=0}^{N_2}\ \tilde{\beta}_k b^k.
\end{equation}

To remove the direct effect, we want to limit the coefficients of this surrogate model to those of the fair target $Y_f$. Our aim will then be to achieve $\tilde{\beta}_k \to 0$ and, for $k>0$, $|\tilde{\alpha}_k| < |\gamma_k|$, which leads us to the following loss:

\begin{align}
 \label{eq:cde-loss}
 \mathcal{R}^{\mathrm{CDE}}_{f} = \sum_{k=1}^{\max(N_1, N_2)} I(|\tilde{\alpha}_k| > |\gamma_k|) (\tilde{\alpha}_k - \gamma_k)^2  + \sum_{k=0}^{N_2} \tilde{\beta}_k^2
\end{align}

In order for our loss to be differentiable, we require that the coefficients $\tilde{\alpha}_k$ and $\tilde{\beta}_k$ are estimated as functions which can be differentiated w.r.t. all the $\tilde{Y}_i$. In our experiments, we used Ordinary Least Square (OLS) regression. Although a wide array of models can be used to define $\mathcal{R}^{\mathrm{CDE}}_{f}$, we recommend that the selected model is collapsible~\cite{Vansteelandt2010, Mood2010} in order to give the correct causal interpretation to the coefficient. 

We emphasize how computing the OLS regression coefficients has a complexity that scales quadratically with the number of covariates. This means that, especially in cases where the number of covariates $\mathbb{X}$ is large, as is often the case in today's applications, our mean field solution allows for a dramatic speed-up.

\section{Experiments}
\label{experiments}
We evaluated our algorithms on three binary classification tasks using a synthetic dataset, which we included for illustrative purposes, the UCI adult dataset~\cite{UCI}, and a commercial credit-risk dataset~\footnote{Private to Experian}. We briefly describe these datasets in Section \ref{sec:datasets} and provide a summary in Table~\ref{tab:data}.

For each dataset, algorithm and loss, we computed results by sweeping the regularization parameter $\lambda$ from $0$ to $0.975$ using a step size of $0.025$. We evaluated every model trained by its fairness, as measured by SPD, and accuracy (or precision, for the credit-risk models). For the synthetic dataset, we evaluated our results using only logistic regression, while on the other two datasets we used both logistic regression and XGBoost~\cite{XGBoost2016}.

For the synthetic dataset, we show results for both the SPD loss [Eq. (\ref{eq:spd-loss})], and CDE loss [Eq. (\ref{eq:cde-loss})], comparing how the results highlight the differences in worldviews that these losses entail for a problem for which we know the causal story. Contrastingly, for the adult and credit-risk we only display results for the CDE loss.

We employed logistic regression with $L_1$ loss to estimate the propensity scores $b$. We also mention that, for higher values of $\lambda$, we found it beneficial to first pre-train our models using small values of $\lambda$ until convergence, and only after that slowly increase $\lambda$ up to the desired value.

We note that our losses are twice differentiable and we supply the diagonal of the Hessian during training to the XGBoost algorithm in every case.
\subsection{Datasets}
\label{sec:datasets}
\begin{table}[t]
\centering
\begin{tabular}{||c | c | c | c ||}
\hline
Dataset & \makecell{n. rows} & \makecell{n. features} &  \makecell{SPD} \\
\hline\hline
Synthetic & 100,000 & 16 & 0.54  \\ 
UCI Adult & 48,842 & 9 & 0.20  \\
Credit-risk & 71,809 & 58 & 0.15   \\
\hline
\end{tabular}
\caption{Summary of datasets included}
\label{tab:data}
\end{table}
\begin{figure*}[t!]
\centering
  \includegraphics[width=0.8\linewidth]{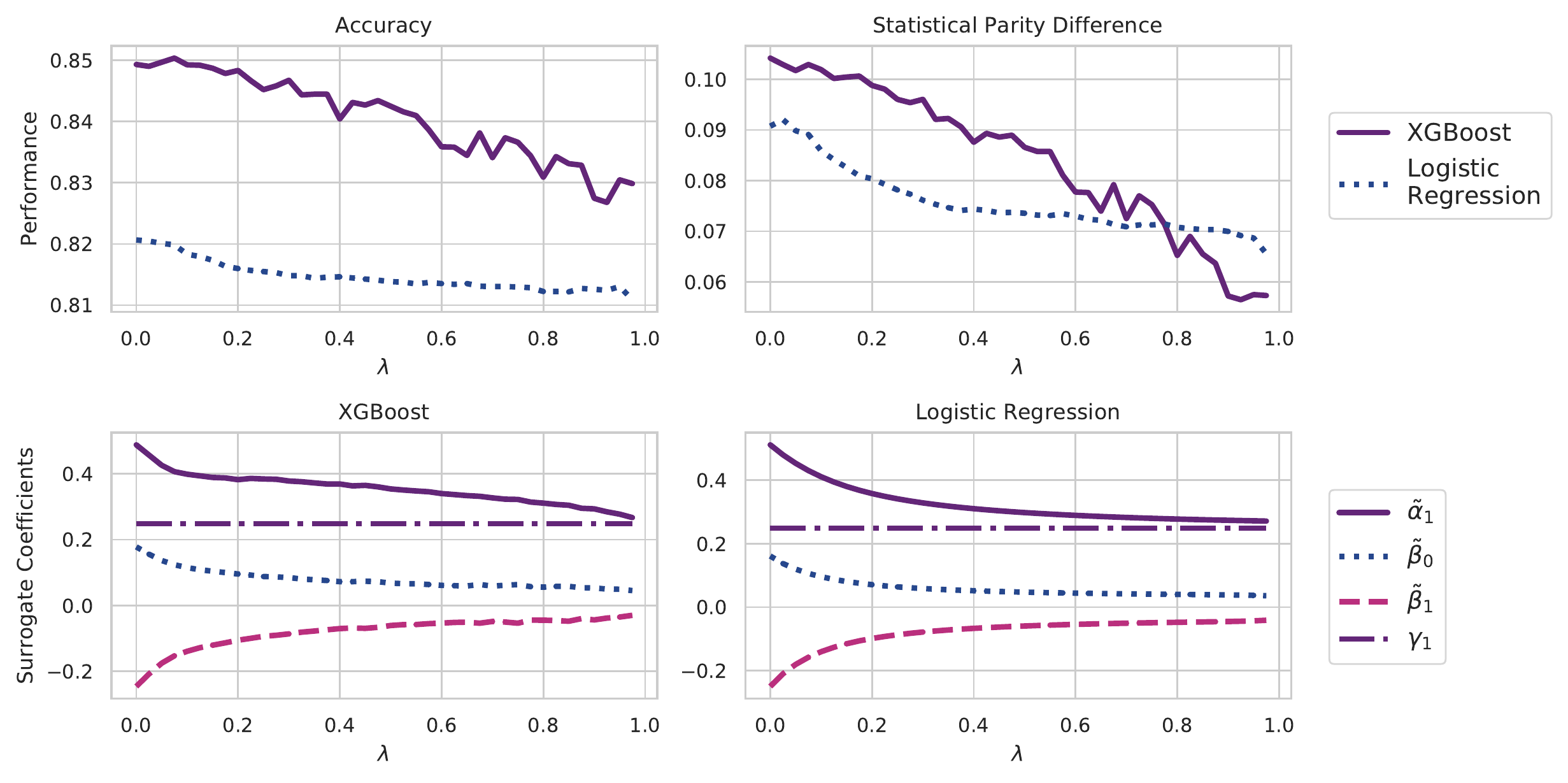}
 \caption{\label{fig:adult}
Results for the UCI Adult dataset employing XGBoost and Logistic regression models modified through the loss of Eq. (\ref{eq:cde-loss}). Results are plotted against the regularization strength $\lambda$. For these experiments, we employed $N_1=N_2=1$. Top panel: Accuracy (left) and SPD (right) using a threshold equal to $0.5$. Bottom panel: surrogate coefficients [see Eq. (\ref{eq:y_hat_b_z})] estimated on the test set for the XGBoost  (left) and Logistic Regression (right) models.}
\end{figure*}

\subsubsection{Synthetic Data}
\label{sec:synth-ds}
We generated a synthetic dataset mimicking the graph of Fig. (\ref{fig:graphs}). We sampled the protected attribute out of a Bernoulli distribution and three sets of covariates: ``safe'' covariates $\mathbb{X}_s \sim \mathcal{N}(0, 1)$, ``proxy'' covariates and ``indirect effect'' covariates sampled from $\mathbb{X}_p, \mathbb{X}_i \sim \mathcal{N}(Z, 1)$. We then define the log-odds of the binary target as $S_Y = 0.25 \mathbf{w} \cdot(\mathbb{X}_i + \mathbb{X}_s) + 1.25 Z$. Here, $\mathbf{w}$ is a vector of ones. We finally sample $Y \sim \mathrm{Bern} \left( \frac{1}{1 + \mathrm{e}^{-S_Y}} \right)$.
For our experiments, we used $10$ safe, $4$ indirect effect and $2$ proxy variables.
\subsubsection{UCI Adult Dataset}
For this dataset, the goal is to predict whether a person will have an income below or above $50K\ \mathrm{USD}$. In this dataset, we are interested in removing gender bias and so we consider gender as our protected attribute. Furthermore, we excluded the following additional sensitive attributes: race, marital status, native country and relationship from our models. The other covariates primarily relate to financial information, occupation and education.

\subsubsection{Private Credit-Risk Dataset}
In this dataset, we are trying to infer the probability that a customer will not default on their credit given curated information on their current account transactions. Here, we're interested in removing bias related to age. We binarized the age variable dividing our examples in two groups, an ``older'' group of people over 50 and an ``younger'' group of people under 50.

\subsection{Results}

Results for the synthetic dataset and logistic regression models are shown in Fig. \ref{fig:synth}. We observe that, as we sweep $\lambda \to 1$, the accuracy of the CDE loss is generally higher than that of the SPD loss, while the SPD is higher. This can be explained by the fact that SPD loss enforces a far more stringent view on fairness than the CDE loss, as it removes both average direct and indirect effects. The SPD loss converges to an almost ideal performance in its target fairness metric. In the bottom panel, which is our main result for this dataset, we show coefficients of the logistic regression model plotted against the regularization strength $\lambda$. We find that the coefficients of the variables that were constructed to be correlated both with the target $Y$ and the protected attribute $Z$ (indirect effect variables) and of those correlated with $Z$ alone (proxy variables) become smaller as the fairness regularization increases. Furthermore, we observe that the coefficient changes with $\lambda$ reflect the different worldviews described in Section \ref{background}. The CDE loss is very faithful to the original causal story (see Section \ref{sec:synth-ds}), where the causal sampling coefficients of the fair and indirect effect variables are the same and the coefficients of the proxy variables are zero. Contrastingly, the SPD loss cause the coefficients to converge to values very different from the ones of the data generating distribution, which in this worldview is in itself deemed ``unfair''.

\begin{figure*}[t!]
\centering
  \includegraphics[width=0.8\linewidth]{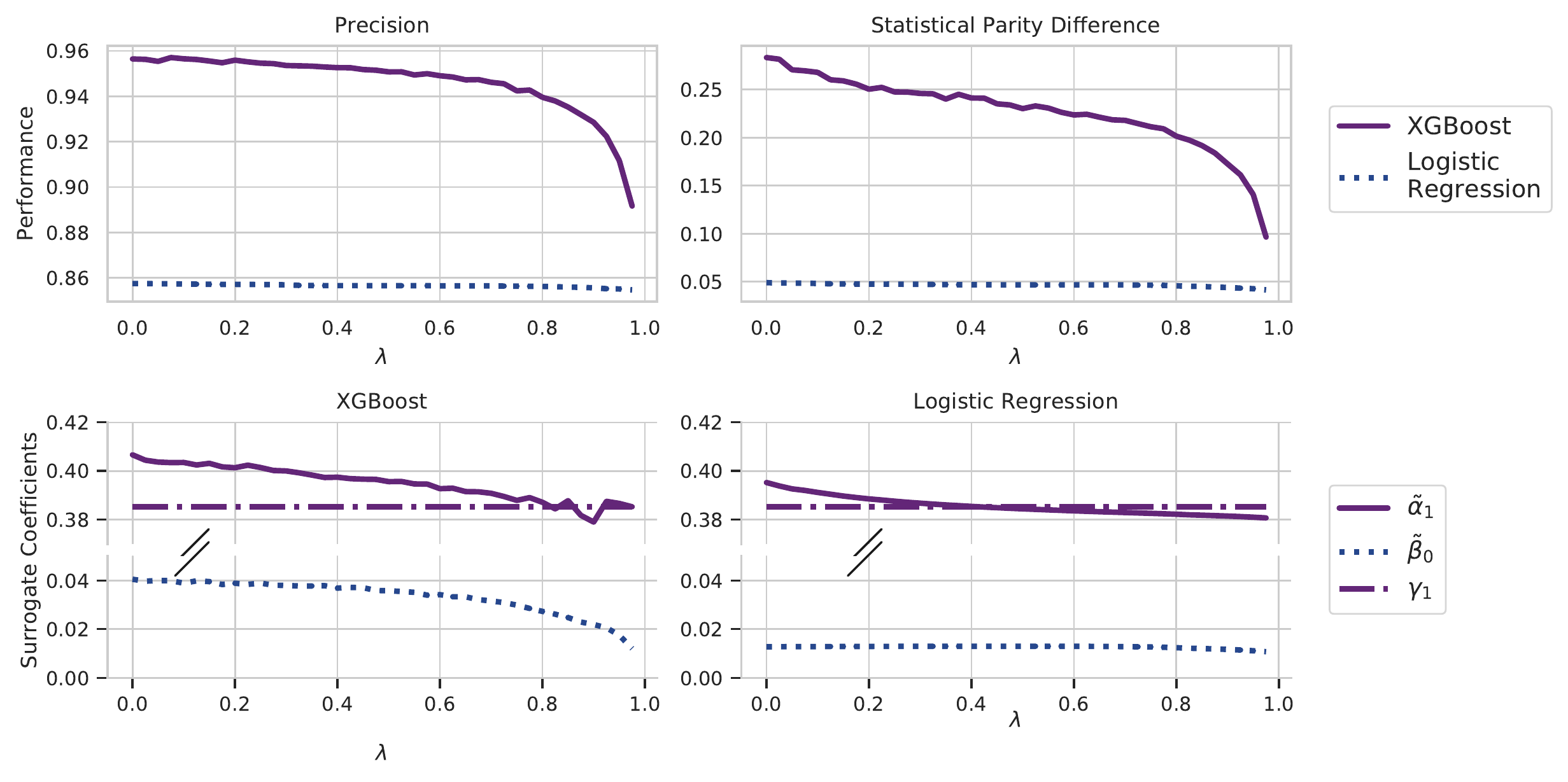}
 \caption{\label{fig:creditrisk}
Results for the credit-risk dataset employing XGBoost and Logistic regression models modified through the loss of Eq. (\ref{eq:cde-loss}). Results are plotted against the regularization strength $\lambda$. For these experiments, we employed $N_1=1, N_2=0$. Top panel: Accuracy (left) and SPD (right) using a threshold equal to $0.85$. Bottom panel: surrogate coefficients [see Eq. (\ref{eq:y_hat_b_z})] estimated on the test set for the XGBoost  (left) and Logistic Regression (right) models.}
\end{figure*}

Results for the UCI Adult dataset are shown in Fig. \ref{fig:adult}. For this dataset, referring to Eq. (\ref{eq:cde-loss}), we used $N_1 = N_2 = 1$. For the XGBoost models, we observe that, as $\lambda$ is increased, accuracy drops by a tolerable amount, i.e. from roughly $0.85$ at $\lambda=0$ to roughly $0.83$ at $\lambda=0.975$. At the same time, SPD is reduced from $0.10$ to $0.06$. Results for logistic regression are more nuanced: accuracy goes from $0.82$ to $0.81$, while SPD decreases from $0.09$ to $0.07$. In the bottom panel we show how the coefficients of the surrogate model [see Eq. (\ref{eq:y_hat_b_z})] change as $\lambda$ increases. We notice how $\tilde{\alpha}_1$ converges to its target $\gamma_1$ while the $\tilde{\beta}_i$ coefficients approach zero.

Finally, in Fig. \ref{fig:creditrisk} we show results for the credit-risk dataset. Here we used $N_1 = 1, N_2 = 0$. Also, we tailored our results to the specificity of the credit-risk problem, where usually people are approved when their default probability is very low. We therefore used a threshold of $0.85$ and, since we are more interested in the cost of false positives than the one of false negatives, we evaluate the performances of the model using precision. Results for XGBoost show that precision went from $0.96$ to $0.89$, while SPD went from $0.28$ to $0.09$. Logistic regression did not seem particularly affected by the regularization. In this case, precision went from $0.86$ to $0.85$, while SPD dropped from $0.05$ to $0.04$. To explain this observation, we conjecture that logistic regression does not have enough statistical capacity to absorb the direct effect when it's not directly exposed to the protected attribute $Z$. Results for the surrogate coefficients (bottom panel) show how the XGBoost models display generally good convergence.

\section{Related Work}
\label{sec:related}
There has been significant advancement in the areas of incorporating fairness into machine learning algorithms and the role of causality in fairness.

\textit{Training fair models}: In Ref.~\cite{IBM2018}, a fair neural network was built through the use of an adversarial model that tries to predict the sensitive attribute from the model outputs. Similarly, Ref.~\cite{FNN2018} developed statistical fairness regularizations to debias neural networks. The form of these regularizations restricted them to neural networks. Ref.~\cite{Goel2018} incorporate convex regularizations directly into training a fair logistic regression, but this approach relies on empirical weights to represent historical bias and is directly related to proportionally fair classification rather than a traditional fairness measure. Other approaches have instead posed the problem as one of constrained optimization while others still have used multiple models to remove bias~\cite{NaiveBayesFair2010, PredRemover2012}.  Here, we propose novel regularizations that are applicable to any model trained using gradient-based optimizers and are designed to target both the classic and causal metrics directly in the scores.

\textit{Causal Fairness}: Several works have recently addressed the problem of tackling some definition of counterfactual fairness~\cite{Kusner2017, CausalReasoningDiscrim2017}. In Refs.~\cite{ZhangDirectEffect, FairInference2018, Chiappa2018, CFGAN} the problem of identifying and removing path specific effects is studied. Those papers consider generative (or partly generative ~\cite{FairInference2018}) models. Since direct effects are particular cases of path specific effects, the scope of those works is somewhat bigger than ours but, crucially, they do not provide a generic method for incorporation of such effects into standard discriminative machine learning models. Our proposal also benefits from being model-agnostic. Furthermore, to compare our worldview with these path-specific works, we borrow an example from \cite{FairInference2018}. In this example, the influence of gender on hiring outcomes for a white collar job might be considered fair through a variable such as education, and unfair through a variable such as physical strength. Here, we argue that for most practical purposes our definition of fair covariates $\mathbb{X}$ (see Section \ref{sec:counterfactual-fairness}) should suggest that physical strength should not be selected among the $\mathbb{X}$ as it easily arguable that it itself is discriminatory attribute and is irrelevant to the problem: one might not want to discriminate between stronger and weaker people for this particular application, regardless on the effect that gender has on it. This combination of what variables are deemed fair individually irrespective of the pathway in conjunction with the CDE makes our worldview particularly novel and applicable in many industrial settings.

\section{Conclusions}
\label{conclusions}
In this work, we extended the literature by proposing a new definition of fairness that focuses on the removal of the controlled direct effect and is causal in nature. Incorporating causal effects into notions of fairness is crucial~\cite{Berkeley}, and we argue that our definition is intuitive and general. We demonstrated how to enforce it through the use of a regularization term. Our solution is particularly appealing with respect to existing ones because it is applicable to any model that uses gradient-based optimization, including popular discriminative models. We exemplified our approaches on three datasets using XGBoost and logistic regression. In all cases, our framework allowed for a realistic trade-off between fairness and predictive performance.

\section{Acknowledgements}
We are indebted to C. Dhanjal, F. Bellosi, G. Jones and L. Stoddart for fruitful discussions. We also wish to acknowledge Experian Ltd and J. Campos Zabala for supporting this work.
\appendix
\section{Balancing scores and Mediation Effects}
In this Appendix, we wish to justify Eq. (\ref{eq:CDEb}), (\ref{eq:rubin}) and (\ref{eq:yf_nie}). Eq. (\ref{eq:CDEb}) is proven below.
\label{app:cde}
\begin{theorem}
 If Assumptions 1 and 2 hold, then, if $b(\mathbb{X})$ is a balancing score, the second line of Eq. (\ref{eq:CDEb}) holds.
\end{theorem}
\begin{proof}
In order to prove our hypothesis, it is sufficient to show that
\[
\mathbb{E}[Y | b, Z] = \frac{\int_{\mathcal{X}_b} \mathrm{d}\mathbb{X}\  p(\mathbb{X}) \mathbb{E}[Y | \mathbb{X}, Z]}{\int_{\mathcal{X}_b} \mathrm{d}\mathbb{X}\  p(\mathbb{X})}
\]
We have:

\begin{align}
\label{eq:theorem}
& \mathbb{E}[Y | b, Z] \\ \nonumber
&= \int \mathrm{d}\mathbb{X}\ \mathbb{E}[Y | b, Z, \mathbb{X}] p(\mathbb{X} | b, Z) \\ \nonumber
&= \int \mathrm{d}\mathbb{X}\ \mathbb{E}[Y | b, Z, \mathbb{X}] \frac{P(Z | \mathbb{X}, b) p (\mathbb{X} | b)}{P(Z | b)} \\ \nonumber
&= \frac{\int_{\mathcal{X}_b} \mathrm{d}\mathbb{X}\ \mathbb{E}[Y | b, Z, \mathbb{X}] \frac{P(Z | \mathbb{X}, b) p (\mathbb{X})}{P(Z | b)}}{\int_{\mathcal{X}_b} \mathrm{d}\mathbb{X}\  p(\mathbb{X})} \\ \nonumber
&= \frac{\int_{\mathcal{X}_b} \mathrm{d}\mathbb{X}\  p(\mathbb{X}) \mathbb{E}[Y | \mathbb{X}, Z]}{\int_{\mathcal{X}_b} \mathrm{d}\mathbb{X}\  p(\mathbb{X})} \nonumber
\end{align}

Where line 4 follows from $p (\mathbb{X} | b) = I(\mathbb{X} \in \mathcal{X}_b)\ p(\mathbb{X}) / \int_{\mathcal{X}_b} \mathrm{d}\mathbb{X}\  p(\mathbb{X})$. The final line follows from the definition of the balancing score.
\end{proof}

In order to justify Eqs. (\ref{eq:rubin}) and (\ref{eq:yf_nie}) we also need to prove the following:

\begin{theorem}
Under the assumptions of Theorem 1, the following equations hold:

\begin{align}
 \mathbb{E}_{\mathbb{X}}[\mathrm{CDE}(\mathbb{X})] &= \int \mathrm{d}b\ \mathrm{MFCDE}(b) p(b) \label{eq:theor-1} \\
 \mathrm{NDE}(Z^*) &= \int \mathrm{d}b\ \mathrm{MFCDE}(b) p(b|Z=Z^*), \label{eq:theor-2}
\\ 
 &\forall Z^* \in \{ 0, 1 \} \nonumber
\end{align}
 \begin{proof}
  The proof of Eqs. (\ref{eq:theor-1}) and (\ref{eq:theor-2}) are very similar. We'll prove Eq. (\ref{eq:theor-2}) and leave the proof of Eq. (\ref{eq:theor-1}) to the reader. It is sufficient to show that, for each value of $Z', Z^* \in \{0, 1\}$ we have
\[
\begin{aligned}
&\int \mathrm{d}\mathbb{X}\  p(\mathbb{X}) \mathbb{E}[Y | \mathbb{X}, Z=Z'] =\\
&\int \mathrm{d}b\  p(b|Z=Z^*) \mathbb{E}[Y | b, Z=Z'].
\end{aligned}
\]
The rest of the thesis follows straightforwardly from the definitions of CDE [Eq. (\ref{eq:CDE})], MFCDE [Eq. (\ref{eq:CDEb})] and NDE [Eq. (\ref{eq:NDE})].
We have:
\[
\begin{aligned}
&\int \mathrm{d}b\ \mathbb{E}[Y | b, Z=Z']  p(b|Z=Z^*) \\ \nonumber
&= \int \mathrm{d}b\  \mathrm{d}\mathbb{X}\ \mathbb{E}[Y | b, Z=Z', \mathbb{X}] p(\mathbb{X} | b, Z=Z') p(b|Z=Z^*) \\ \nonumber
&= \int \mathrm{d}b\  \mathrm{d}\mathbb{X}\ \Big\{ \mathbb{E}[Y | b, Z=Z', \mathbb{X}] \\
& \hspace{60px} \frac{P(Z=Z' | \mathbb{X}, b) p (b|\mathbb{X}) p(\mathbb{X})}{P(Z=Z' | b)p(b)} p(b|Z=Z^*)\Big\} \\ \nonumber
&= \int \mathrm{d}b\  \mathrm{d}\mathbb{X}\ \Big\{ \mathbb{E}[Y | b, Z=Z', \mathbb{X}] \\
& \hspace{60px} \frac{P(Z=Z' | \mathbb{X}, b) \delta(b - b(\mathbb{X})) p(\mathbb{X})}{P(Z=Z' | b)p(b)} p(b|Z=Z^*) \Big\} \\ \nonumber
&= \int \mathrm{d}\mathbb{X}\ \Big\{ \mathbb{E}[Y | Z=Z', \mathbb{X}] \\
 & \hspace{60px} \frac{P(Z=Z' | \mathbb{X}, b(\mathbb{X})) p(\mathbb{X})}{P(Z=Z' | b(\mathbb{X}))p(b(\mathbb{X}))} p(b(\mathbb{X})|Z=Z^*)) \Big\} \\ \nonumber
&= \int \mathrm{d}\mathbb{X}\ \mathbb{E}[Y | Z=Z', \mathbb{X}] p(\mathbb{X}) \frac{P(Z=Z^*|\mathbb{X})}{P(Z=Z^*)} \\ \nonumber
&= \int \mathrm{d}\mathbb{X}\ \mathbb{E}[Y | Z=Z', \mathbb{X}] p (\mathbb{X} | Z=Z^*) \\ \nonumber
\end{aligned}
\]
where line 3 follows from two applications of Bayes' rule, $\delta(\bullet)$ in line 4 is Dirac's delta distribution, line 5 follows from integrating $b$ out, line 6 again from Bayes' rule and the definition of a balancing score, which entails $P(Z|b(\mathbb{X})) = P(Z|(\mathbb{X})$; the final line follows straightforwardly, from a reverse application of Bayes' rule.
 \end{proof}
\end{theorem}

Eq. (\ref{eq:yf_nie}) is then derived from Eqs. (\ref{eq:theor-2}) and (\ref{eq:ate-cie}) with minimal algebra. 
\bibliography{fairness}
\bibliographystyle{ACM-Reference-Format}
\newpage
\section*{Reproducibility Notes}
\begin{printonly}
Supplementary materials are available in the online version of this paper.
\end{printonly}
\subsection*{Scope}
 In this document, we wish to include a few details that should help the reader reproduce our results. In particular, we will be focused on the results for the synthetic dataset and the UCI adult dataset (Figures 3 and 4, respectively, of the main text). The credit-risk results will not be reproducible due to the private nature of the dataset.
 \subsection*{Technology}
 We used python v. 3.6, with the main packages employed being scikit-learn v. 0.21, pandas v. 0.24, numpy v. 1.16 and xgboost v. 0.82.
\subsection*{Computation of gradients and diagonal hessians}
In this section we will give insights on how to compute gradients and diagonal hessians w.r.t. the probability scores of the model evaluated in each training example. Specifically, given the model scores on the training set $\tilde{Y}_i$, we want therefore to compute $\partial \mathcal{R}_f^{\mathrm{CDE}} / \partial \tilde{Y}_i$ and $\partial^2 \mathcal{R}_f^{\mathrm{CDE}} /\partial \tilde{Y}_i^2$ for all the examples $i$ in the training set. Using these derivatives, XGBoost models can be trained directly\cite{XGBoost2016}, and any parametric model can be trained evaluating $\nabla_{\theta} \mathcal{R}_f^{\mathrm{CDE}}$ through the chain rule as usual. Since these computations are quite trivial for the loss of Eq. (\ref{eq:spd-loss}), we will only focus on the loss of Eq. (\ref{eq:cde-loss}).

The approach we used was to compute the regression coefficients of both Eqs. (\ref{eq:y_b_z}) and (\ref{eq:yf}) using OLS regression. Given the coefficients of Eq. (\ref{eq:y_b_z}), the loss of Eq. (\ref{eq:spd-loss}) can be seen as a function of the coefficients of Eq. (\ref{eq:yf}), that we'll collectively denote $\mathbf{\zeta}$:

\begin{equation}
 \mathbf{\zeta} = [\tilde{\alpha}_0, ..., \tilde{\alpha}_{\max(N_1, N_2)}, \tilde{\beta}_0, ..., \tilde{\beta}_{N_2}] .
\end{equation}

We want to express $\mathbf{\zeta}$ as a function of the $\tilde{Y}_i$, so that computing the desired derivatives will be straightforward to the reader.

We define a regression matrix $\Gamma$ whose rows $\Gamma_i$ for each example in the training set are given by:

\begin{equation}
 \Gamma_i = [1, b_i, b_i^2, ..., b^{\max(N_1, N_2)}, Z_i, b Z_i, ..., b_i^{N_2} Z_i],
\end{equation}

where $b_i$ and $Z_i$ are the balancing scores and protected attribute for the $i$-th example. The coefficients $\mathbf{\zeta}$ in OLS regression are then easily found to be:

\begin{equation}
 \mathbf{\zeta}(\mathbf{\tilde{Y}}) = (\Gamma \Gamma^T)^{-1} \Gamma^T \mathbf{\tilde{Y}}
\end{equation}

where $\mathbf{\tilde{Y}}$ is the vector of the scores evaluated in each training example.
\subsection*{Data splits, pre-processing and hyperparameters}
For the UCI adult dataset, we used the official train/test split and used one-hot encoding for the categorical variables. For the synthetic dataset, we sampled $10^5$ points and used $33\%$ for testing. We used scikit-learn's StandardScaler to preprocess all the datasets.

For our logistic regression models we did not use any hyperparameter apart from $\lambda$. For our XGBoost models, with reference to the python API, we used the default parameters with the exception of $\mathrm{reg\_lambda} = 10$, $\mathrm{learning\_rate} = 0.1$ and $\mathrm{max\_depth} = 2$.
\subsection*{Propensity Scores}
Propensity scores were computed using scikit-learn's LogisticRegression and we used GridSearchCv from the same library to search for the inverse $L_1$ penalty term $\mathrm{C} \in \{10^{-2}, 10^{-1}, 1, 10, 10^2 \}$. We used 5-Fold cross-validation scored by accuracy.
\subsection*{Training procedures}
We trained all our models using a double early stopping procedure as follows. For logistic regression, we initialized the weights of every model to the ones of a logistic regression trained with scikit-learn with the default parameters (we used the ``liblinear'' solver) and subsequently used gradient descent using our modified loss. We also defined an early stopping set, which was the whole training set in the logistic regression case and a separate holdout set for XGBoost. The holdout set was derived using $33 \%$ \textdiscount of the training set (we used a numpy.random seed fixed at $123$ throughout).

Then, for each value of $\lambda$, we used the following procedure:
\begin{enumerate}
 \item Define $\lambda^* = \min(\lambda, 0.3)$
 \item Train using the loss $ \mathcal{L}_f = (1 - \lambda^*) \mathcal{L}_{o} + \lambda^* \mathcal{R}_{f}$ until $\mathcal{L}_{o}$ does not improve on the early stopping set for $5$ steps
 \item Increase $\lambda^*$ linearly over 50 steps from $\min(\lambda, 0.3)$ to $\lambda$
 \item Train using the loss $ \mathcal{L}_f = (1 - \lambda) \mathcal{L}_{o} + \lambda \mathcal{R}_{f}$ until $\mathcal{L}_f$ does not improve on the early stopping set for $20$ steps
\end{enumerate}

\end{document}